\documentclass{article}
\usepackage{arxiv}
\usepackage{amssymb}
\usepackage[utf8]{inputenc} % allow utf-8 input
\usepackage[T1]{fontenc}    % use 8-bit T1 fonts
\usepackage{hyperref}       % hyperlinks
\usepackage{url}            % simple URL typesetting
\usepackage{booktabs}       % professional-quality tables
\usepackage{amsfonts}       % blackboard math symbols
\usepackage{nicefrac}       % compact symbols for 1/2, etc.
\usepackage{microtype}      % microtypography
\usepackage{lipsum}		% Can be removed after putting your text content
\usepackage{graphicx}
\usepackage{natbib}
\usepackage{comment}
\usepackage{doi}
\usepackage{amsmath}
\usepackage{physics}
\usepackage{wrapfig}
\usepackage{bm}
\usepackage{xcolor}
\usepackage{forest}
\usepackage[linesnumbered,ruled]{algorithm2e}
\usepackage{tikz-cd}

\newenvironment{proof}{\paragraph{Proof:}}{\hfill$\square$}
\newtheorem{theorem}{Theorem}

\newtheorem{lemma}[theorem]{Lemma}

\newtheorem{definition}[theorem]{Definition}

\title{Multiplayer Information Asymmetric Bandits in Metric Spaces}
\date{} 					% Or removing it

\author{ William Chang \\
	Department of Applied Mathematics\\
	University of California, Los Angeles\\
	Los Angles, CA \\
	\url{https://williamc.me/} \\
	\And
     Aditi Karthik \\
	Department of Psychology\\
	University of California, Los Angeles\\
	Los Angles, CA \\
	\And
}
\renewcommand{\headeright}{}
\renewcommand{\undertitle}{}
\hypersetup{
pdftitle={},
pdfsubject={},
pdfauthor={William Chang},
pdfkeywords={},
}

\begin{document}
\maketitle

\begin{abstract}
  In recent years the information asymmetric Lipschitz bandits  In this paper we studied the Lipschitz bandit problem applied to the multiplayer information asymmetric problem studied in \cite{chang2022online, chang2023optimal}. More specifically we consider information asymmetry in rewards, actions, or both. We adopt the CAB algorithm given in \cite{kleinberg2004nearly} which uses a fixed discretization to give regret bounds of the same order (in the dimension of the action) space in all 3 problem settings. We also adopt their zooming algorithm \cite{ kleinberg2008multi}which uses an adaptive discretization and apply it to information asymmetry in rewards and information asymmetry in actions. 
\end{abstract}

\section{Introduction}
The multi-armed bandit problem is a quintessential problem in reinforcement learning when an agent is met with a (finite) set of actions to choose from. Taking an action results in a fixed $1$-subgaussian stochastic reward in the interval $[0, 1]$, and the goal is to maximize the expected rewards across some number of rounds. Since the rewards are bounded, so are their means, so we can suppose each arm $a$ has some true mean $\mu_a \in [0, 1]$. Maximizing the expected rewards then amounts to pulling arms with the highest $\mu_a$ as frequently as possible. 

There has been interest in a different kind of bandit setting where the actions are no longer assumed to be finite. this is called the Lipschitz bandit setting where actions are now indexed by some subset of the Euclidean space. In this setting algorithms designed for the finite set of actions (e.g. UCB \cite{lai1985asymptotically}) no longer work as the UCB algorithm necessitates pulling each arm of the infinitely many arms at least once. However, in the Lipschitz bandit setting, for two actions $a$ and $b$ one can bound the absolute difference between their true means $|\mu_a - \mu_b|$ by some constant (the Lipschitz constant) multiplied by the distance between $a$ and $b$. The Lipschitz condition is explicitly given in equation \eqref{lipschitz}. Empirically this means that if one has a good estimate for the mean of a particular arm, then it also acts as a good estimate for the arms nearby.

In the single-player setting the classical lipshictz bandits problem was studied by \cite{kleinberg2004nearly}. They solved this problem using a fixed discretization where the action space is evenly split into a finite number of actions \texttt{CAB}. On this discretized set they apply the standard UCB algorithm \cite{lai1985asymptotically}. \cite{kleinberg2008multi} proposes a more efficient algorithm called the zooming algorithm. In comparison the the fixed discretziation, the zooming algorithm discretizes regions that are more likely to have arms with better rewards more finely. They prove that this algorithm achieves $O(T^{\frac{d+1}{d+2}})$ where $d$ is the covering dimension of the space with a matching lower bound.

There has recently been a rise in multiplayer bandits which have seen applications in radio networks. In many of these multiplayer bandit works they either assume a graph structure that allows communication with other players or they don't consider joint actions. The line of multiplayer work that we are interested in this paper is in the information asymmetric bandit setting \cite{chang2022online, chang2023optimal}. In those settings, each the players are each given their own set of actions. At every round, the players pull their own action simultaneously which results in a \emph{joint} action. The action space is thus exponential in size with respect to the number of players. In this setting, we are interested in two main kinds of information asymmetry. The first is information asymmetry in actions where players receive the same rewards but are unable to observe the actions of the other players.

\paragraph{Our contribution} 
In this paper we propose the first algorithm for the multiplayer information asymmetric Lipschitz bandits. We proposed 3 problems A, B, and C that involved information asymmetry in actions, rewards, and both, respectively. We show how to use uniform discretization as well as the multiagent algorithms proposed recently in \cite{chang2022online, chang2023optimal} to solve information asymmetry these 3 problems. We then modify the zooming algorithm to solve information asymmetry in only actions and only rewards. We were able to prove regret bounds on the same order as in the single-player case for Problems A and B while having a regret bound that is nearly optimal for problem C. 

\subsection{Related Works}

\paragraph{Lipschitz Bandits} There have been many works in Lipschitz bandits. Here we outline a couple of them.

Some papers study the Lipschitz bandits where the Lipschitz constant is not known. One of the earliest papers in this field is  \cite {jones1993lipschitzian} Solves Lipschitz optimization, finding the global minimum without the Lipschitz constant by searching for all possible constants. Modifies Shubert's algorithm to prevent overemphasis on global search and reduce computational. Another instance is in \cite{bubeck2011lipschitz} They solve the Lipschitz bandit problems with unknown Lipschitz values (by estimating the Lipschitz value before exploration-exploitation). Using worst-case regret bounds to minimize optimal order even when the Lipschitz value is unknown. \cite{lu2010contextual} Studies Lipschitz bandit problems of a metric space setting. Propose query-ad-clustering algorithm, spliting the Y subset into ‘good ads’ and ‘bad ads’ and using lower bounds to find the optimal regret. \cite{kleinberg2019bandits}: Solve the lipschitz bandit problem in general metric spaces. Prior work only focuses on $[0, 1]$ under the metric $\ell_1^{1/d}$. They try to answer which metric spaces have sublinear regret and try to understand what is the best possible bound for regret on a given metric space. \cite{bubeck2011x} solves invariant Lipschitz bandit problems, with sets of arms and rewards invarient under a set of known transformations. Propose UniformMesh-N, uses group orbits to integrate side observations into a uniform discretization algorithm. In \cite{magureanu2014lipschitz} they solve the discrete Lipschitz bandit problem using lower bounds for regret. Propose OSLB, with a regret matching the lower bound, and CKL-UCB, and show finite time analysis of the regrets of the above algorithms. Iin \cite{lu2019optimal} they solve Lipschitz bandit problems with randomized reward distributions, modeling heavy-tailed distributions. Propose SDTM, standard discretization with truncated mean, as they cannot use the classic UCB.

\cite {krishnamurthy2020contextual} studies bandit learning in both a lipschitz context and and a regret bound without continuity assumptions. They define smoothed regret, with policies that map context to distribution over actions. Proposes multiple algorithms to solve for regret. In \cite{wang2020towards} they solve the standard Lipschitz bandit problems differently (link tree-based methods to Gaussian processes). They propose TreeUCB, which keeps a UCB index for every region, the only difference is at every round they will "fit a tree" based on previous rewards. Finally, in \cite{feng2022lipschitz} they solve Lipschitz bandit problems with reward observations given to players in batches. Propose BLiN, using a zooming algorithm to find the optimal arm with a lower bound.

\paragraph{Multiplayer works}
Here we outline some of the works in multiplayer bandits. 

In cooperative bandits the goal is to identify the best arm from a set of common arms between players. There is a graph that represents the communication structure between players. This was first introduced by \cite{awerbuch2008competitive}. Susequently, various strategies have been proposed; $\epsilon$-greedy \cite{szorenyi2013gossip}, gossip UCB \cite{landgren2016distributed}, accelerated gossip UCB \cite{martinez2019decentralized}, and adopting a leader \cite{wang2020optimal}. \cite{bar2019individual} studied this problem in the adversarial setting and proposed a strategy where the followers follows EXP3. Another line of works allows the players to observe the rewards of their neighbors at each time step based on their proximity in the graph \cite{cesa2016delay} There has also has been some asynchronous works where only a subset of players are active each round \cite{bonnefoi2017multi, cesa2020cooperative}.

In the collision setting multiple players select the same arm, and there is a collision and all players cannot collect rewards from the arm. Here they do not consider joint arms. The extension of this to the Lipschitz setting is in \cite {proutiere2019optimal}, where they propose DPE, decentralized parsimonious exploration, an algorithm that requires little communication and allows players to maximize their cumulative reward.

Competing bandits were first introduced by \cite{liu2020competing}. This model is similar to the collision setting but there are preferences between players. When several players pull the same arm, only the top-ranked player gets the reward. In this work, they propose a centralized CUB algorithm, where each player sends their UCB index to a central agent. In \cite{cen2022regret} they showed that a logarithmic optimal regret is reachable for this algorithm if the platform can also choose transfers between the players and arms. In \cite{jagadeesan2021learning} consider the stronger notion of equilibrium where the agents also negotiate these transfers. In \cite{liu2020competing}, propose an ETC algorithm reaching a logarithmic optimal regret without any transfer requiring the knowledge of the gaps, however, \cite{sankararaman2021dominate} extend this algorithm without such knowledge. \cite{liu2021bandit} also propose a decentralized UCB algorithm with a collision avoidance mechanism

\section{Preliminary}\label{sec:preliminary}

We study the $L$-Lipschtiz multiplayer bandit problem where the joint actions are taken to be over $[0, 1]^{md}$. This action space is decomposed into a direct product of $m$ copies of $[0, 1]^d$. Each copy of $[0, 1]^d$ is the action space of one particular player. At every round $t$, each player will pick an arm in $[0, 1]^d$ simultaneously without communication. Their joint action will be the direct product of the individual actions, thus belonging in $[0, 1]^{md}$.  

We shall assume the true means of two joint arms $\bm{a}$ and $\bm{a'}$ satisfy the following Lipschitz condition
\begin{equation}\label{lipschitz}
    |\mu_{\bm{a} - \mu_{\bm{a'}}}| \leq L\norm{\bm{a} - \bm{a}'}
\end{equation}
where the norm on the right-hand side of the inequality can be any norm on the Euclidean space, and $L > 0$.  For convenience let $B(\bm{a}, r)$ be the ball of arms centered at $\bm{a}$ with radius $r$. Let $\mu_{\bm{a}}$ be the true mean of the rewars for joint arm $\bm{a}$. Furthermore, let $\mu^*$ be the true mean of the optimal arm. The definition of regret we shall use is 
\begin{equation}
    R_T =T\mu^* - \sum_{t=1}^T \mathbb{E}[X_{\bm{a}_t}] =  \sum_{t=1}^T \Delta_{\bm{a}_t}
\end{equation}

where $\Delta_{\bm{a}} = \mu^* - \mu_{\bm{a}}$. 

In this paper, we study the following types of information asymmetry that have been studied in previous settings such as \cite{chang2022online, chang2023optimal}. In all of these settings, they can agree on a strategy beforehand, but once the learning process begins, they are no longer allowed to communicate or send each other information. 

\paragraph{Problem A} Asymmetry in Actions. In this setting, each player is unable to observe the other player's actions but they observe the same reward. 

\paragraph{Problem B} Asymmetry in Rewards. In this setting, each player can observe the other player's actions, but they each obtain their own i.i.d. realization of the reward. They cannot observe other player's realizations. 

\paragraph{Problem C} Asymmetry in both Actions and Rewards
In this setting, each player is unable to observe the other player's actions, and they obtain their own i.i.d realization of the reward. 

Note that even in Problems B and C where each player receives their own reward, since the rewards are i.i.d., they will experience the same regret as regret is calculated in \emph{expectation}. 

\section{Uniform Discretization}\label{sec:theorem} 

In this section we take the regret bounds and algorithms from \cite{chang2022online, chang2023optimal} and show how we can apply these algorithms to a setting with continuous (and therefore infinite) actions. The algorithms in the aforementioned works are regret bounds for the information asymmetries described by Problems A, B, C in the previous section but for the classical MAB with joint actions. 

\subsection{Lipschitz Regret to Discrete Bandits Regret}

We can take a Lipschitz reward setting and apply the MAB algorithms by discretizing the action space. Each player will discretize one of their intervals $[0, 1]$ into $K$ even parts. that is, we put markers at $\{0, \frac1K, \frac2K,...,\frac{K-1}{K}\}$. We take the cartesian product of this discretization $\{0, \frac1K, \frac2K,...,\frac{K-1}{K}\}^d$ to obtain a discretization of the action space for each player in $[0, 1]^d$.

Denote $\bm{a}_t \in [0, 1]^{Md}$ to be the continuous joint action taken at time $t$. $\mu^*$ is the optimal continuous joint action across $[0, 1]^M$. We define $R_T$ the regret on the Lipschitz bandit space $[0, 1]^{Md}$. This is the regret we want to minimize. More explicitly, 
$$R_T = \mu^* T - \sum_{t=1}^T \mu_{a_t}$$

Since we are going to discretizing our space it is also useful to define $R_K^\pi(T)$, the regret on the discretized space with $K$ arms. Let $\bm{a}_t \in \{\frac{1}{N}, \frac{2}{N},...,1\}^{Md}$ is the discrete joint action taken at time $t$. There are a total of $N^{Md}$ discrete arms. $\mu^*_K$ is the optimal discrete action across $ \{\frac{1}{N}, \frac{2}{N},...,1\}^{Md}$

$$R_K^\pi(T) = \mu_K^* T - \sum_{t=1}^T \mu_{a_t}$$

Note where $\mu_K^*$ is the optimal action with respect to the discretized action space and is thus upper bounded by $\mu^*$.

\begin{lemma}\label{lemma:convert}
    Suppose the regret of a policy $\pi$ on an action space of $K$ arms with horizon $T$ is $R_K^\pi(T)$. Then regret $R_T$ on the environment with continuous action space $[0, 1]^{Md}$ satisfies
    $$R_T \leq \frac{TL}{K} +R_K^\pi(T)$$
\end{lemma}

In each of these settings it is useful to define, for each joint arm $\bm{a}$ the following UCB index $\operatorname{UCB}_{\bm{a}}(t)$. This is defined as 
$$\operatorname{UCB}_{\bm{a}}(t) = \widehat{\mu}_{\bm{a}}(t) + \underbrace{\sqrt{\frac{6\log T}{n_{\bm{a}}(t)}}}_{:= \epsilon_{\bm{a}}(t)}$$

Here $\epsilon_{\bm{a}}(t)$ can be seen as the "error" term, which is an upper bound for how far the empirical mean is from the true mean. Note that this quantity is only dependent on the number of times an arm has been pulled and the horizon. More specifically, it decreases as the number of times an arm has been pulled increases. 

We now present the proof of the above lemma    \begin{proof}
    $R_T$ is the regret on a Lipschitz bandit space $[0, 1]^d$
        $$R_T = \mu^* T - \sum_{t=1}^T \mu_{a_t}$$

        Suppose $a^*$ is the optimal arm corresponding to the mean $\mu^*$. Furthermore, suppose $a^* \in [a, a+ \frac1k]$ where $a$ is an action in the discretized set. 
        $$R_T = (\mu^* T - \mu_{a}T) + (\mu_{a}T -  \mu_K^* T) + (\mu_K^* T  - \sum_{t=1}^T \mu_{b_t}) %(\sum_{t=1}^T \mu_{b_t} - \sum_{t=1}^T \mu_{a_t}) 
        $$
        
        %The  $\mu^*$ of the bandit is always between two discrete arms, $\mu$ and $\mu_k+1$
        $$|a^* - a| \leq \frac1K \implies |\mu^* T - \mu_a T| \leq \frac{LT}{K}$$ 
        %where k is the number of arms.
       % $\frac{1}{k} < \frac{T*L}{k} $ where L is the Lipschitz constant, 
        %$$(\mu^* T - \mu_K^* T) \leq \frac{LT}{K}$$
We also know that $\mu_K^* \geq \mu_a$ so that 
$$(\mu_{a}T -  \mu_K^* T)  \leq 0$$
        $R_K^\pi(T)$, or middle part/second subequation of the above, is the regret on the discretized space with $K$ arms. 
        $$R_K^\pi(T) = \mu_K^* T - \sum_{t=1}^T \mu_{b_t}$$
    
      We can plug our upper bounds into our regret decomposition to obtain,
\begin{align}
R_t&= (\mu^* T - \mu_{a}T) + (\mu_{a}T -  \mu_K^* T) + (\mu_K^* T  - \sum_{t=1}^T \mu_{b_t}) \\
&\leq \frac{LT}{K} + R_K^\pi(T).
\end{align}
        
    \end{proof}

\subsection{Asymmetry in Arms (Problem A)}
In this section, we propose Problem A Asymmetry in actions. This setting describes a multiplayer problem with where each player is unable to view the other players actions. All players get the same reward depending on the actions of both players. 

In a single player scheme, players can view their selected arm and the reward they receive from the distribution, using this information to create an optimal balance between exploitation and exploration. In the multiplayer setting, certain information is hidden from the players - the other player's actions, the other player's reward, or both. In \cite{chang2022online, chang2023optimal} they develop algorithms that deal with each form of information asymmetry. In particular, due to the lack of communication, the players have to develop a coordination scheme. 

To apply the algorithms in finite armed bandits to arms in continuous space \cite{NIPS2004_b75bd27b}, discretized the action space into a finite number of arms, then applied the algorithms from finite armed bandits to the discretized space. The Lipschitz continuity of the actions ensures that with a fine enough discretization scheme, the unsampled arms will not be too far from the sampled ones. In this paper, we do the same thing but with the multiagent algorithms that have been developed to deal with information asymmetry.

We now propose mCAB-A used for asymmetry in actions. In this algorithm, all players will discretize their action space into $
\{0, \frac1K,...,1\}^d$ and play on this $K^d$ action space. This gives $K^{Md}$ joint actions. The idea is that each player will keep track of a UCB-index for each \emph{joint} arm (not just the arms they can pull). However, this presents a challenge in that each player may not know which joint arm to update the UCB index for as they cannot observe the actions taken by the other players. 

To resolve this, the players all agree to an ordering of all the joint actions beforehand, as they are unable to view which arm the other players received. All players will receive the same reward from the selected arm's reward distribution. As the UCB index depends on the empirical mean and the number of times an action has been pulled, if they are coordinated up to time $t$, for each joint arm, all the players will have the same UCB index. If the players simply go for the joint arm with the highest UCB index, they will only have an issue if two arms have the same UCB index. To resolve this we have the following ordering on all the joint arms. 

\begin{definition}\label{def:order}
    Number the players $1,...,m$ and the $K$ individual actions, and consider each set of joint action $\bm{a}$ as an $m$ digit number with each digit corresponding to the joint action. Call this base $K$ number $N_{\bm{a}}$. For joint action $\bm{a}, \bm{b} \in \mathcal{A}$, we say that $\bm{a} < \bm{b}$ if $N_{\bm{a}} < N_{\bm{b}}$. 
\end{definition}

Therefore, if all the players agree to pull the smallest arm as defined in the ordering above in the case of a tie, the players can infer the actions of the other players without having to observe them. The algorithm is stated more explicitly in Algorithm \ref{algo:mCAB-A}.

\begin{algorithm}
\caption{\texttt{mCAB-A} for asymmetry in actions}        
 \textbf{Input:} $T, M, d \in \mathbb{N}$

Each player will discretize their actions space as $\{0, \frac1K, \frac2K,...,1\}^d$, where $K$ is given in equation \eqref{K Prob A}
 
Run \texttt{mUCB} on discretized joint action space \cite{chang2022online}. More explicitly, 
\For{$t \leq K_{\max}$}{
    Player $P_i$ will start from his arm $1$ and successively pull each arm $K_{i+1}\cdots K_{M}$ times before moving to the next arm. They will repeat this entire epoch $K_1\cdots K_{i-1}$ times. 
}
		
\For{$t > K_{\max}$}{
Player $P_i$ chooses arm $a_i ^*(t)= \arg \max_{a_i} \left(\max_{a_{-i}}\operatorname{UCB}^i_{a_i,a_{-i}}(t)\right)$, which corresponds to player $i$ picking the $i$th component of $a^*(t)$ that maximises the index $\operatorname{UCB}_a(t)$.  In case of a tie between say $a$ and $a'$, they pick corresponding components of $a$ such that $a < a'$, where the order relation is as specified in Definition \ref{def:order}.
		
  Player $P_i$ updates the UCB index $\operatorname{UCB}^i_a(t+1)$ for arm $a$ setting $\delta = \frac{1}{T^2}$ with the received reward $X^i_{a^*(t)}(t)$.
}

\label{algo:mCAB-A}
\end{algorithm}

We can now present the regret bound for the Problem A algorithm. 
\begin{theorem}\label{theorem:A}
The regret of \texttt{mCAB-A} is given by 
$$R_T =  O\left(T^\frac{2Md + 1}{2Md+2}L^{\frac{Md}{Md+1}}\\(\log T)^{\frac{1}{2(Md+1)}}\right)$$
\end{theorem}

\begin{proof}
Suppose each player discretizes their actions space $[0, 1]$ into $K$ equal portions. Using theorem 2 of \cite{chang2022online} we know that the regret $R_K(T)$ of \texttt{mUCB} on $K^M$ joint actions is given by 
    $$R_K(T) = O\left(K^{Md} \sqrt{T\log(T)}\right)$$
    Using Lemma \ref{lemma:convert} we conclude that the Lipschitz regret is given by 
    $$R_T \leq \frac{TL}{K^{Md}} + O\left(K^{Md} \sqrt{T\log T}\right)$$
    We can minimize the upper bound above by using AM-GM inequality to conclude that the minimum is attained when the two terms are of the same order. That is
    \begin{equation}\label{K Prob A}
        K = \frac{T^{\frac{1}{2(Md+1)}}L^{\frac{1}{M+1}}}{(\log T)^{\frac{1}{2(Md+1)}}}
    \end{equation}

    Therefore, our Lipschitz regret is given by 
    $$R_T = O\left(T^\frac{2Md + 1}{2Md+2}L^{\frac{Md}{Md+1}}\\(\log T)^{\frac{1}{2(Md+1)}}\right)$$
\end{proof}

\subsection{Asymmetry in Rewards  (Problem B)}
In this section, we propose an algorithm for Problem B Asymmetry in rewards. Here, we are unable to follow the same approach as in Problem A since the rewards are i.i.d., and thus the players will observe different UCB indices. This means that players will likely mis-coordinate (especially in the early rounds) where they will have different empirical means. See \cite{chang2022online} for the ineffectiveness of the simple coordination scheme used in problem A for problem B. 

However, we can use the observability of actions to improve on the algorithm proposed in Problem A. This algorithm is inspired from \texttt{mUCB-Intervals} from \cite{chang2023optimal}. As with problem A we discretize the action space uniformly. For each joint action $\bm{a}$ we know with high probability that 
\begin{equation}
   \mu_{\bm{a}} \in  I_{\bm{a}}(t) := \left(\widehat{\mu}_{\bm{a}}(t) -\underbrace{\sqrt{\frac{6\log T}{n_{\bm{a}}(t)}}}_{:= \epsilon_{\bm{a}}(t)},  \widehat{\mu}_{\bm{a}}(t) +\sqrt{\frac{6\log T}{n_{\bm{a}}(t)}}\right)
\end{equation}
This means that under the "good event" that the above equation holds, we know that if $I_{\bm{b}}(t)$ is above and disjoint from $I_{\bm{a}}(t)$ \footnote{That is, $\widehat{\mu}_{\bm{a}}(t) + \epsilon_{\bm{a}}(t) < \widehat{\mu}_{\bm{b}}(t) - \epsilon_{\bm{b}}(t)$}, then we can say with high probability that arm $\bm{b}$ is better than arm $\bm{a}$. Therefore, we can eliminate arm $\bm{a}$ and not pull it anymore. 

Inspired by this principle, we propose the following algorithm. Each player keeps track of a desired set that contains all the arms that could still \emph{potentially} be the best arm. Like with Problem A, there is an ordering on all the joint arms as follows. 

\begin{equation}\label{eq:flow}  \begin{tikzcd}\bm{a}_1 \to \bm{a}_2\to \cdots  \to \bm{a}_{K^{Md}} \arrow[bend left = 10,start anchor={[xshift= 10 ex, yshift = -1.5ex]},end anchor={[xshift= -9ex, yshift = -1.5ex]}]\end{tikzcd}\end{equation}

The players will pull the arms that still \emph{belong in the desired set} in this ordering. The elimination procedure goes as follows. Suppose the next arm in the desired set according to the above ordering is $\bm{a}$. Furthermore, suppose player $i$ observes that there exists some other joint arm $\bm{b}$ such that $I_{\bm{b}} > I_{\bm{a}}$. Then, player $i$ will intentionally pull an arm that is not equal to $\bm{a}[i]$. The other players, being able to observe actions, will see that action $\bm{a}$ should be next in line, but the resulting joint action is different than $\bm{a}$. Therefore, they know to remove arm $\bm{a}$ from the desired set. The ordering stays the same, but the arms no longer in the desired set get skipped over. The observability and the pre-agreed upon ordering essentially allow players to maintain the same desired set without explicit communication features in the environment. The pseudocode is given in Algorithm \ref{algo:mCAB-B}.

\begin{algorithm}
\label{algo:mCAB-B}\caption{\texttt{mCAB-B} for asymmetry in actions}        
 \textbf{Input:} $T, m, d \in \mathbb{N}$

Each player will discretize their actions space as $\{0, \frac1K, \frac2K,...,1\}^d$, where $K$ is given in \eqref{K Prob A}
 
Run \texttt{mUCB-Intervals} on discretized joint action space \cite{chang2023optimal}

Each player $P_i$ has all the joint arms in their \emph{desired sets}. All the players will agree on the ordering of the joint arms.

\For{$t = 1,\ldots,(K^d)^M$}{
Each player $i$ will pull each joint action $\bm{a}$ once in the order they have decided in advance, and update $I_{\bm{a}}^i$. 
}

\For{$t = K^M+1,\ldots, T$}
{
    Each player $i$ identifies the next arm considered $\bm{c}_{t}$ in the desired set based on the arm pulled in the previous round (see flowchart in \eqref{eq:flow}).\
    
    \eIf{exists player $i$ and joint action $\bm{a'}$ such that $I_{\bm{a'}}^i$ is above and disjoint from $I_{\bm{c}_t}^i$}{
    Player $i$ will not pull $\bm{c}_t[i]$ to inform the other players that he will remove $\bm{c}_{t}$ from his desired set. \
    }
    {
    Each player $i$ pull $\bm{c}_t[i]$.\ 
    }
    Each player $i$ observes the actions of other players to determine the joint action taken $\bm{a}_t$ at that step. They observe their own i.i.d. reward and update their $I_{\bm{a}_t}^i$. \
    
    \If{$\bm{a}_t \neq \bm{c}_t$}{
    All players eliminate $\bm{c}_{t}$ from their desired set whilst maintaining the same ordering of the remaining arms.
    }
}
\end{algorithm}

We can now present the regret bound for the Problem B algorithm. 
\begin{theorem}
    The regret of \texttt{mCAB-B} is given by 
    $$R_T =  O\left(T^\frac{2M d+ 1}{2Md+2}L^{\frac{Md}{Md+1}}\\(\log T)^{\frac{1}{2(Md+1)}}\right)$$
\end{theorem}
    \begin{proof}
In accordance to Theorem 2 \cite{chang2023optimal} we know that the regret on \texttt{mUCB-Intervals} is the same as that in \texttt{mUCB} of \cite{chang2022online}. Therefore using Lemma \ref{lemma:convert}, we obtain the same regret bound as in Theorem \ref{theorem:A}.
    \end{proof}
    
\subsection{Asymmetry in both arms and rewards (Problem C)}

In this section we propose an algorithm for Problem C which describes a multiplayer problem with both an asymmetry in action and reward. Note that this means the algorithms for Problem A and Problem B no longer work as they rely on players having same rewards and being able to observe the actions respectively. 

 We therefore apply the \texttt{mDSEE} algorithm from \cite{chang2022online} to the Lipschitz setting, which can be seen as a improved explore then commit algorithm. As before, all players will independently choose an arm within the discretized action, and agree to an ordering of joint actions beforehand. The algorithm happens in phases.At phase $n$, (which will occur when the round number is at a power of $2$), the players will uniformly explore all the joint actions $f(n)$ times. We call this the exploration phase. After that, the players will take all the samples from the exploration phases (including the previous ones), and commit to the arm with the highest empirical mean. This is the committing phase, which stops once the round number is again a power of $2$. Note that since the exploration occurs at exponentially increasing intervals, this only contributes $f(\log T)$ regret which is small for a slowly growing function $f$. On the other hand, the probability that \emph{all} players will pull the corresponding components to the optimal arm will increase since they are accumulating more and more samples prior to each successive committing phase. The pseudocode is given in Algorithm \ref{algo:mCAB-C}.

\begin{algorithm}\caption{\texttt{mCAB-C} for asymmetry in actions}        
 \textbf{Input:} $T, m, d \in \mathbb{N}$
 
Each player will discretize their actions space as $\{0, \frac1K, \frac2K,...,1\}^d$, where $K$ is given by \eqref{K Prob C}
 
Run \texttt{mDSEE} on discretized joint action space \cite{chang2022online}

 Pick a monotonic function $f(n): \mathbb{N} \to \mathbb{N}$ such that $\lim_{n \to \infty}K(n) = \infty$. 
		
First let $\lambda = 1$
		
Player $P_i$ will start from his arm $1$ and successively pull each arm $f(n)K_{i+1}\cdots K_M$ before moving to the next arm. He will repeat this entire epoch $K_1\cdots K_{i-1}$ times.
		
Player $P_i$ will calculate the sample mean $\mu_j^i(t)$  of the rewards they see for each $M$-tuple $a$ of arms.

Player $P_i$ will choose the $M$-tuple arm with the highest sample mean and commit to his corresponding arm for up until the next power of $2$. In case of a tie, pick randomly. 
		
When $T = 2^n$ for some $n \geq \lfloor\log_2(f(1)K_1\cdots K_M)\rfloor+1$, Player $P_i$ with repeat steps (5) - (8), incrementing $n$ by $1$. 
\label{algo:mCAB-C}
\end{algorithm}

Using the results from \cite{chang2022online}, we note that on $K$ actions the regret of \texttt{mDSEE} is given by $O(f(\lfloor \log T\rfloor)\log T)$ where $f: \mathbb{N} \to \mathbb{N}$ is some function that satisfies $\lim_{n \to \infty} f(n) = \infty$. 

\begin{theorem}
   Using $f(n) = \sqrt{n}$, The (gap independent) regret bound for \ref{algo:mCAB-C} is given by 
    $$R_T  \leq O\left(T^\frac{2Md + 1}{2Md+2}L^{\frac{Md}{Md+1}}\\(\log T)^{\frac{1}{(Md+1)}}\right)$$
\end{theorem}

\begin{proof}
    Suppose each player discretized their action space $[0, 1]$ into $K$ equal partitions. Using Theorem 2 of \cite{chang2022online} we know that the regret $R_K(T)$ on $K^{Md}$ joint actions is given by 
    $$R_T \leq \frac{TL}{K^{Md}} + O(K^{Md} \log T\sqrt{T} )$$

    Notice the increased power of $\log T$. Setting the two terms equal to each other to mimize the above, we can set,

    \begin{equation}\label{K Prob C}
        K = \frac{T^{\frac{1}{2(Md+1)}}L^{\frac{1}{M+1}}}{(\log T)^{\frac{1}{(Md+1)}}}
    \end{equation}

    Therefore, our Lipschitz regret is given by 
    $$R_T = O\left(T^\frac{2Md + 1}{2Md+2}L^{\frac{Md}{Md+1}}\\(\log T)^{\frac{1}{(Md+1)}}\right)$$
\end{proof}

\section{A more adaptive zooming algorithm}
In this section, we provide a multiplayer extension of the zooming algorithm from \cite{kleinberg2008multi} applied to Problems A and B. The zooming algorithm is superior to the fixed discretized schemes from the previous section, since it gives a finer discretization to regions where the optimal arm is more likely to reside.  

\subsection{Problem A}
In this section we propose a zooming algorithm that is suited for Problem A, information asymmetry in actions. The challenge in naively applying \texttt{mZoom-A} lies at the step where a new arm is activated. In the single-player algorithm, each time a region is uncovered, the the players are to add an additional arm. However in a continuous space, it is unclear which arm they should add. The space is the complement of a union of balls which is not convex and can take on very strange shapes. To deal with this challenge note that the players have balls of the same size, and thus we can implicitly coordinate by discretizing the action space using the following doubling technique.  To this end, we propose the following definition. 

\begin{definition}\label{def:2^n}
    Given a unit cube $[0, 1]^m$ and some uniform discretization $\left(\{\frac{i}{2^n}\}_{i=0}^n\right)^m$, we say we double the cube if we cut each segment in half so our discretization is now $\left(\{\frac{i}{2^{n+1}}\}_{i=0}^{n+1}\right)^m$. In essence, we have created $2^m$ times more cubes each with $\frac{1}{2^m}$ of the volume. 
\end{definition}

Let the set of arms that are being played at any given time be the set of \emph{active} arms. In each of these sets, we will have balls center at these arms with radii $r_{\bm{a}}(t) = \frac{1}{L} \epsilon_{\bm{a}(t)}$. 

\begin{algorithm}
\caption{\texttt{mZoom-A}}        
 \textbf{Input:} $T, M, d \in \mathbb{N}$

Each player will activate $(\frac12,...,\frac12) \in A$ where $A$ is the active set

\For{$t = 1,...,T$}{
If there is a region that is uncovered by the balls with radius $r_{\bm{a}}(t) = L \epsilon_{\bm{a}(t)}$ we will double the region (as in definition \ref{def:2^n}) until at least $1$ point in our new discretization is in an unoccupied region. Add the arm with the smallest order relation as specified in Definition \ref{def:order} into $\mathcal{A}$. 

Player $P_i$ chooses arm $a_i ^*(t)= \arg \max_{a \in \mathcal{A}} \left(\max_{a_{-i}}\operatorname{UCB}^i_{a_i,a_{-i}}(t)\right)$ which corresponds to player $i$ picking the $i$th component of $a^*(t)$ that maximizes the index $\operatorname{UCB}_a(t)$.  In case of a tie between say $a$ and $a'$, they pick corresponding components of $a$ such that $a < a'$, where the order relation is as specified in Definition \ref{def:order}.
		
  Player $P_i$ updates the UCB index $\operatorname{UCB}^i_a(t+1)$ for arm $a$ setting $\delta = \frac{1}{T^2}$ with the received reward $X^i_{a^*(t)}(t)$.
}

\label{algo:mZoom-A}
\end{algorithm}

The region is covered with balls that hold arms. If an uncovered region contains an arm, then a new ball is placed to cover the uncovered arm. As the regions are a discretized space, open arms are found via this discretization. If there are no open regions in the current discretized space, the space is discretized exponentially by 2. The discretized space is divided in half until an uncovered region is found. If there are multiple uncovered arms, the players can define an order to cover the arms, as there is a finite number of arms to choose from. The players will infer which arms the players will pick (including the newly created arms), explained in Definition \ref{def:2^n}. 

The following gives us a regret bound for the regret of Algorithm \ref{algo:mZoom-A}

\begin{theorem}\label{thm:ZoomA}
    The regret of \texttt{mZoom-A} in  Algorithm \ref{algo:mZoom-A} is $$R(T)=O\left(T^{\frac{Md+1}{Md+2}}(c \log T)^{\frac{1}{Md+2}}\right)$$
\end{theorem}

The proof is in section \ref{zoomAproof} of the appendix.

\subsection{Problem B}
In this section, we propose an adaptive algorithm Problem B information asymmetry in rewards. We use an algorithm similar to Algorithm \ref{algo:mCAB-B}. Similar to the zooming algorithm in the previous section, we put balls of radii $r_{\bm{a}}(t)$ at every open uncovered region. To deal with the fact that there are infinitely many options for which arm to select, we used the same doubling trick in Definition \ref{def:2^n}. We follow a similar principle to algorithm \ref{algo:mCAB-B} used for Problem B. MOre specifically, amongst the active arms, the players will pull them in a specific following Definition \ref{def:order}. The elimination procedure remains the same. That is, suppose the next joint arm in the ordering is $\bm{a}$. If there exists player $i$ that observes another active joint arm whose interval is above and disjoint from $\bm{a}$, the player $i$ will intentionally pull something different from arm $\bm{a}[i]$. The other players will observe this and eliminate $\bm{a}$ from their desired set. This essentially allows the players to maintain the same desired set without a need for communication. The pseudocode is given in Algorithm \ref{algo:mZoom-B}. 

   \begin{algorithm}
\label{algo:mZoom-B}\caption{\texttt{Zoom-B} for asymmetry in actions}        
 \textbf{Input:} $T, m, d \in \mathbb{N}$

Each player will activate $(\frac12,...,\frac12) \in A$ where $A$ is the active set, and add it to the desired set as well. 

\For{$t = 1,...,T$}{
If there is an arm that is uncovered, double the region until at least $1$ point in the discretization is in an unoccupied region. Add the arm with the smallest order relation as specified in Definition \ref{def:order} as the last element in the desired set.

Pull the new arm the same number of times as all the other arms in the desired set.

If there is another uncovered region, repeat steps 4 and 5 until the entire region is covered.

    Each player $i$ identifies the next arm considered $\bm{c}_{t}$ in the desired set based on the arm pulled in the previous round.\
    
    \eIf{exists player $i$ and joint action $\bm{a},  \in \mathcal{A}_t$ such that $\widehat{\mu}_{\bm{c}} + 2\epsilon_{\bm{c}} < \widehat{\mu}_{\bm{a}} - \epsilon_{\bm{a}} $}{\label{eliminate}
    Player $i$ will not pull $\bm{c}[i]$ to inform the other players that he will remove $\bm{a}$ from his desired set. \
    }
    {
    Each player $i$ pull $\bm{c}_t[i]$.\ 
    }
    Each player $i$ observes the actions of other players to determine the joint action taken $\bm{a}_t$ at that step. They observe their own i.i.d. reward and update $\widehat{\mu}_{\bm{a}_t}$, $\epsilon_{\bm{a}_t}$. \
    
    \If{$\bm{a}_t \neq \bm{c}_t$}{
    All players eliminate $\bm{c}_{t}$ from their desired set whilst maintaining the same ordering of the remaining arms. The radius of the ball centered at this arm would then remain unchanged for the remainder of the learning. \
    }
}
\end{algorithm}

Note that step 6 must terminate after a finite number of steps since every ball in the desired set has the same radius (meaning the radii cannot be arbitrarily small).

This algorithm is interesting since during the process new active arms are created while at the same time, arms already active are also being eliminated. Note that once an active arm $\bm{a}$ is eliminated, this also eliminates the arms distance $r_{\bm{a}}(t)$ from $\bm{a}$. Here $t$ is the round that $\bm{a}$ was eliminated. This is made explicit in step \ref{eliminate} of Algorithm \ref{algo:mZoom-B}, when the players eliminate an arm $\bm{a}$ from the desired set, they are in fact eliminating the ball of arms centered at $\bm{a}$ with radius $r_{\bm{a}}(t)$. We have to ensure that the ball that gets eliminated doesn't contain the optimal arm. 
\begin{lemma}\label{lemma:keep_best}
   The optimal arm, under the good event $G$, never gets eliminated.  

   \begin{proof}
       Suppose that $\bm{c}$ was the arm that's getting eliminated that is, and and suppose that $\bm{a}$ is such that 
       $$\widehat{\mu}_{\bm{c}} + 2\epsilon_{\bm{c}} < \widehat{\mu}_{\bm{a}} - \epsilon_{\bm{a}} $$

       Then all the arms $\bm{b} \in B(\bm{c}, r_{\bm{c}})$ satisfy the following (under the good event $G$). 
       \begin{align}
           \mu_{\bm{b}} &\leq \mu_{\bm{c}} + L\norm{\bm{c} -\bm{b}}\\
           &\leq \mu_{\bm{c}} + Lr_{\bm{c}}\\
           &= \mu_{\bm{c}} + L\frac{1}{L} \epsilon_{\bm{c}}\\
           &= \widehat{\mu}_{\bm{c}} + \epsilon_{\bm{c}} + L\frac{1}{L} \epsilon_{\bm{c}}\\
           &< \widehat{\mu}_{\bm{a}} - \epsilon_{\bm{a}}\\
           &= \mu_{\bm{a}}
       \end{align}

       Therefore $\bm{b}$ cannot be the optimal arm for its mean is less than or equal to another arm. 
       
   \end{proof}
\end{lemma}

Using the above lemma we can prove the following regret bound 
\begin{theorem}
The regret of \texttt{mZoom-B} in  Algorithm \ref{algo:mZoom-B} is $$R(T)=O\left(T^{\frac{Md+1}{Md+2}}(c \log T)^{\frac{1}{Md+2}}\right)$$    
\end{theorem}
Note here that the constant in this theorem is different than the one in Theorem \ref{thm:ZoomA}. The difference is illustrated in equations \eqref{eq:num_pulls_probA} and \eqref{eq:num_pulls_probB}. The proof is given in section \ref{zoomBproof} of the appendix. 

\section{Conclusions}
In this paper we propose the first algorithm for the multiplayer information asymmetric Lipschitz bandits. We proposed 3 problems A, B, and C that involved information asymmetry in actions, rewards, and both, respectively. We show how to use uniform discretization as well as the multiagent algorithms proposed recently in \cite{chang2022online, chang2023optimal} to solve information asymmetry in these 3 problems. We then modify the zooming algorithm to solve information asymmetry in only actions and only rewards. We were able to prove regret bounds on the same order as in the single-player case for Problems A and B while having a regret bound that is nearly optimal for problem C.

\newpage
\bibliography{references}
\bibliographystyle{abbrvnat}

\newpage
\appendix
\section{Important Lemmmas}

We have the following Lemma that serves as a concentration inequality for subgaussian random variables from \cite{lattimore2020bandit}.
\begin{lemma}\label{corollary5.5}
Assume that $X_i-\mu$ are independent, $\sigma$-subgaussian random variables. Then for any $\varepsilon \geq 0$,
$$
\mathbb{P}(\hat{\mu} \geq \mu+\varepsilon) \leq \exp \left(-\frac{n \varepsilon^2}{2 \sigma^2}\right) \quad \text { and } \quad \mathbb{P}(\hat{\mu} \leq \mu-\varepsilon) \leq \exp \left(-\frac{n \varepsilon^2}{2 \sigma^2}\right),
$$
where $\hat{\mu}=\frac{1}{n} \sum_{t=1}^n X_t$.
\end{lemma}

\section{Regret analysis for \texttt{mZoom-A} for asymmetry in actions}\label{zoomAproof}

Consider the good event $G$ for each arm $\bm{a}$ at time $t$ be defined as 
    $$G_{\bm{a}}(t) = \{|\widehat{\mu}_{\bm{a}}(t) - \mu_{\bm{a}}| \leq \epsilon_{\bm{a}}(t)\}$$

Furthermore, let $\mathcal{A}_t$ be the set of activated arms at round $t$ (note that we have $\mathcal{A}_t \subset\mathcal{A}_{t+1}$) and let $G$ be the good event that $G_{\bm{a}}(t)$ holds across any round $t$ and any activated arm. That is,

$$G = \bigcap_{t=1}^T \bigcap_{\bm{a} \in \mathcal{A}_t} G_{\bm{a}}(t)$$

    The following lemma gives us a bound for the probability that this good event holds. 
\begin{lemma}\label{lemma:good}
The probability that $G^c$ holds is upper bounded by $\frac{1}{T}$. That is,
$$P(G^c) \leq \frac{1}{T}$$
\end{lemma}

\begin{proof}
\begin{equation}\label{epsilon}
    \epsilon_{\bm{a}} = \sqrt{\frac{6\log(T)}{n_a(t)}}
\end{equation}

  We first find an upper bound on the probability of the 'bad' event for each joint action $G_{\bm{a}}(t)$. Using lemma \ref{corollary5.5}, we know we have the following, 
    
    \begin{align}
        P(\hat{\mu_a} \geq \mu_a + \epsilon_a(t)) &\leq  e^{-\frac{n_a(t)\left(\sqrt{\frac{6\log(T)}{n_a(t)}}\right)^2}{2}}\\
        &=  e^{\frac{-n_a(t)\cdot 6 \log(T)}{2n_a(t)}} \\
       &= e^{-3\log(T)}\\
    &=  (e^{\log(T)})^{-3} \\ 
     &=  T^{-3}
    \end{align}

    where we used the definition of $\epsilon_{\bm{a}}$ in the first inequality. 
Using Demorgans rule, we now have
    $$G^c =\bigcup_{t=1}^T \bigcup_{\bm{a} \in \mathcal{A}_t} G_{\bm{a}}(t)^c $$

    Therefore, by our probability union bound, we can upper bound the probability of the 'bad' event, 

    \begin{align}
        P(G^c) &= P\left(\bigcup_{t=1}^T \bigcup_{\bm{a} \in \mathcal{A}_t} G_a^c(t)\right)\\
        & = \sum_{t=1}^T \sum_{\bm{a} \in \mathcal{A}_t} P(G_a^c(t))\\
        &=\sum_{t=1}^T \sum_{\bm{a} \in \mathcal{A}_t}\frac{1}{T^3}\\
        &= \sum_{t=1}^T |\mathcal{A}_t|\frac{1}{T^3}\\
        &\leq \sum_{t=1}^T \frac{1}{T^2} \\
        &= \frac{1}{T}
    \end{align}

    where in the 3rd inequality we used our upper bound above for $P(G_{\bm{a}(t))})$. In the 4th equality we used the fact that $|\mathcal{A}_t| \leq T$ because we are only adding at most 1 arm per found to our active set. 
\end{proof}

We can now prove Theorem \ref{thm:ZoomA}
\begin{proof}

Note that the regret decomposition by the \href{https://en.wikipedia.org/wiki/Law_of_total_expectation}{Law of Total Expectation} is, 
$$R_T \leq \mathbb{E}[R_T|G]P(G)+ \mathbb{E}[R_T|G^c]P(G^c)$$

Convince yourself that $\mathbb{E}[R_T|G^c] \leq T$ and $P(G) \leq 1$. You also proved already that $P(G^c) \leq \frac1T$. Therefore, 
$$R_T \leq \mathbb{E}[R_T|G] + 1 $$

We now calculate the regret assuming that the good event $G$ happens. 

Since the algorithm pulls the arm with the highest regret every round, if $\bm{a}_t$ was played, this implies $$\operatorname{UCB}_{\bm{a}_t}(t) \geq \operatorname{UCB}_{\bm{a}_t^*}$$
where $\bm{a}_t^*$ is arm in the active set $\mathcal{A}_t$ that covers the true optimal arm $\bm{a}^*$. 
Now note 
\begin{align}
\operatorname{UCB}_{\bm{a}_t}&\geq \operatorname{UCB}_{\bm{a}_t^*} \\
&\geq \widehat{\mu}_{\bm{a}_t^*} + \epsilon_{\bm{a}_t}(t) + \epsilon_{\bm{a}_t}(t)\\
&\geq \mu_{\bm{a}_t^*} + L\norm{\bm{a_t} - \bm{a^*}}\\
&\geq \mu^*
\end{align}

Furthermore,  
\begin{align}
\operatorname{UCB}_{\bm{a}_t}(t) &= \widehat{\mu}_{\bm{a}_t}(t) + 2 \epsilon_{\bm{a}_t}(t) \\
   &\leq \mu_{\bm{a}_t} + 3 \epsilon_{\bm{a}_t}(t)
\end{align}
Combining the upper and lower bound for $\operatorname{UCB}_{\bm{a}_t}$ above, we conclude that under the good event, for each arm $\bm{a} \in \mathcal{A}$, we have $\Delta_{\bm{a}} \leq 3\epsilon_{\bm{a}}(t)$ for each time $t$. Note that this is true regardless the arm is pulled at that time $t$ or not. 

From the definition of $\epsilon_{\bm{a}_t}(t)$ in equation \eqref{epsilon},

$$\Delta_{\bm{a}} \leq 3 \sqrt{\frac{6\log T}{n_{\bm{a}}(t)}} $$
Rearranging the above, we conclude that 
\begin{equation}\label{eq:num_pulls_probA}
    n_{\bm{a}}(t) \leq \frac{54\log T}{\Delta_{\bm{a}}^2}
\end{equation}

We now proceed as in the proof of Theorem \ref{algo:mZoom-B} but with eqution \eqref{eq:num_pulls_probA} replacing equation \eqref{eq:num_pulls_probB}.
\end{proof}

\section{Regret analysis for \texttt{mZoom-B} for asymmetry in rewards}\label{zoomBproof}

\begin{proof}
We define the following "good" event $G$ where all arms have their true means in the intervals at all times. Explicitly, this is written as 
    
    \begin{equation}
        G = \bigcap_{t=1}^T \bigcap_{i=1}^M \bigcap_{\bm{a} \in \mathcal{A}} \{|\hat{\mu}_{\bm{a}}^i(t) - \mu_{\bm{a}}|<\epsilon_{\bm{a}}(t ,\delta)\}
    \end{equation}
The probability of the complement of this event is already bounded in the proof of Theorem \ref{lemma:good}.

By lemma \ref{eliminate}, the best arm is never in a region tha is eliminated. Therefore at round $t$, we can call $\bm{a}_t^*$ the arm whose ball contains the best arm. As all the arms are pulled in a round-robin fashion in the desired set, if an arm $a$ is at the desired set at time $t$ then $n_{\bm{a}^*}(t) \geq n_{\bm{a}}(t) -1$. It follows that at the time when $n_{\bm{a}}(t) = n_{\bm{a}}$, we have $n_{\bm{a}^*}(t) \geq n_{\bm{a}} - 1$. When each arm in the desired set has been pulled at least twice, this means that $n_{\bm{a}^*}(t) \geq \frac12 n_{\bm{a}}$. Under the good event note that 
$$\widehat{\mu}_{\bm{a}}(t) \leq \mu_{\bm{a}} + \epsilon_{\bm{a}}(t) \implies \operatorname{UCB}_{\bm{a}}(t) \leq \mu_{\bm{a}} + 2\epsilon_{\bm{a}}(t) = \mu_{\bm{a}} + 2\sqrt{\frac{6\log T}{n_{\bm{a}}(t)}}$$
We we recalled the definition of $\epsilon_{\bm{a}}(t)$ give in equation \eqref{epsilon} Similarly, 

$$\operatorname{LCB}_{\bm{a}_t^*}\geq \mu_{\bm{a}_t^*} - 2\epsilon_{\bm{a}_t^*}(t) \geq \mu_{\bm{a}_t^*} - 2\sqrt{\frac{6\log T}{n_{\bm{a}}(t)/2}}$$

Furthermore, note that $$\mu_{\bm{a}^*} \leq \mu_{\bm{a}_t^*} + L r_{\bm{a}}(t) \leq \mu_{\bm{a}_t^*} 
 + \epsilon_{\bm{a}_t^*}(t) \leq \mu_{\bm{a}_t^*} + \sqrt{\frac{6\log T}{n_{\bm{a}}(t)/2}}$$

    Therefore,
     $$\operatorname{LCB}_{\bm{a}_t^*}(t)\geq \mu_{\bm{a}^*} - 3\sqrt{\frac{6\log T}{n_{\bm{a}}(t)/2}}$$

    It follows that if $\bm{a}_t$ is pulled, it is in the desired set so that $\operatorname{LCB}_{\bm{a}_t^*}(t) \geq \operatorname{UCB}_{\bm{a}}(t) $, it follows that we have the following necessary condition 
    $$ \mu_{\bm{a}^*} - 3\sqrt{\frac{6\log T}{n_{\bm{a}}(t)/2}} \leq \mu_{\bm{a}} + 2\sqrt{\frac{6\log T}{n_{\bm{a}}(t)}}$$
Rearranging the above gives 
    \begin{equation}\label{eq:num_pulls_probB}
n_{\bm{a}}(t) \leq (6\sqrt{3} + 2\sqrt{6})^2 \frac{\log T}{\Delta_{\bm{a}}^2}
    \end{equation}
     
For $r>0$, consider the set of active arms whose badness is between $r$ and $2 r$ :
$$
X_r=\{x \in \mathcal{A}_t: r \leq \Delta_{\bm{a}}<2 r\} .
$$

Fix $i \in \mathbb{N}$ and let $Y_i=X_r$, where $r=2^{-i}$. From the bulletpoint above, for any $\bm{a}, \bm{a'} \in Y_i$, we have $d(\bm{a}, \bm{a'})>\frac{1}{3L} \Delta_{\bm{a}}$. If we cover $Y_i$ with subsets of diameter $\frac{r}{3L}$, then arms $\bm{a}$ and $\bm{a'}$ cannot lie in the same subset. Since one can cover $Y_i$ with $N_{\frac{r}{3L}}\left(Y_i\right)$ such subsets, it follows that $\left|Y_i\right| \leq N_{\frac{r}{3L}}\left(Y_i\right)$.

Note that in the following analysis we don't really care of $\Delta_{\bm{a}} = 0$ since we will split the good arms from the bad, the following upper bound will only be used on the bad arms. From equation \eqref{eq:num_pulls_probB}, we have:
$$
R_i(T):=\sum_{\bm{a}\in Y_i} \Delta_{\bm{a}} \mathbb{E}[n_{\bm{a}}(T)] \leq\sum_{\bm{a}\in Y_i} \Delta_{\bm{a}} \cdot \frac{O(\log T)}{\Delta_{\bm{a}}^2} \leq \frac{O(\log T)}{\Delta_{\bm{a}}} \cdot N_{r / 3}\left(Y_i\right) \leq \frac{O(\log T)}{r} \cdot N_{r / 3}\left(Y_i\right) .
$$
Pick $\delta>0$, and consider arms with $\Delta(\cdot) \leq \delta$ separately from those with $\Delta(\cdot)>\delta$. Note that the total regret from the former cannot exceed $\delta$ per round. Therefore:
$$
\begin{aligned}
R(T) & \leq \delta T+\sum_{i: r=2^{-i}>\delta} R_i(T) \\
& \leq \delta T+\sum_{i: r=2^{-i}>\delta} \frac{\Theta(\log T)}{r} N_{r / 3}\left(Y_i\right) \\
& \leq \delta T+O(c \cdot \log T) \cdot\left(\frac{1}{\delta}\right)^{d+1},
\end{aligned}
$$
where $c$ is a constant and $d$ is some number such that
$$
N_{r / 3}\left(X_r\right) \leq \frac{c}{r^d} \quad \forall r>0 .
$$

the last inequality then follows from $r = 2^{-i} > \delta$. 

The smallest (infimum) such $d$ is called the zooming dimension with multiplier $c$.
By choosing $\delta=\left(\frac{\log T}{T}\right)^{1 /(Md+2)}$, we obtain
$$
R(T)=O\left(T^{\frac{Md+1}{Md+2}}(c \log T)^{\frac{1}{Md+2}}\right) .
$$

Note that we make this choice in the analysis only; the algorithm does not depend on the $\delta$.

\end{proof}  
\end{document}